\documentclass{article}

\usepackage{microtype}
\usepackage{graphicx}
\usepackage{subfigure}
\usepackage{booktabs} 
\usepackage{amsmath}
\usepackage{amssymb}
\usepackage{stackengine}

\usepackage{tikz}
\usepackage{pgfplots}
\usepgfplotslibrary{fillbetween}
\usetikzlibrary{patterns}
\usetikzlibrary{shapes,arrows,chains}

\usepackage{grffile}

\usepackage{hyperref}

\usepackage[accepted]{icml2018}

\usepackage{siunitx}

\newtheorem{theorem}{Theorem}[section]
\newtheorem{lemma}[theorem]{Lemma}
\newtheorem{definition}[theorem]{Definition}
\newtheorem{proposition}[theorem]{Proposition}

\newenvironment{proof}{\par\noindent{\bf Proof:\ }}{\hfill$\Box$\\[2mm]}

\newif\ifpaper
\papertrue 

\def\RR{\mathbb{R}}
\def\>{\rangle}

\def\Set#1{\left\{ #1 \right\}}
\def\Bigbar#1{\mathrel{\left|\vphantom{#1}\right.}}
\def\Setbar#1#2{\Set{#1 \Bigbar{#1 #2} #2}}
\newcommand{\red}[1]{{\color{red}#1}}

\newcommand{\inner}[1]{\left\langle#1\right\rangle}

\def\bydef{\mathrel{\mathop:}=}

\def\dom{\mathop{\rm dom}\nolimits}
\def\range{\mathop{\rm range}\nolimits}

\def\ker{\mathop{\rm ker}\nolimits}

\def\max{\mathop{\rm max}\nolimits}

\def\ie{\textit{i.e. }}
\def\eg{\textit{e.g. }}

\icmltitlerunning{Neural Networks Should Be Wide Enough to Learn Disconnected Decision Regions}

\begin{document}

\twocolumn[
\icmltitle{Neural Networks Should Be Wide Enough to Learn Disconnected Decision Regions}

\icmlsetsymbol{equal}{*}

\begin{icmlauthorlist}
\icmlauthor{Quynh Nguyen}{a}
\icmlauthor{Mahesh Chandra Mukkamala}{a}
\icmlauthor{Matthias Hein}{b}
\end{icmlauthorlist}

\icmlaffiliation{a}{Department of Mathematics and Computer Science, Saarland University, Germany}
\icmlaffiliation{b}{University of T{\"u}bingen, Germany}
\icmlcorrespondingauthor{Quynh Nguyen}{quynh@cs.uni-saarland.de}
\icmlkeywords{neural networks, classification regions, connected decision regions}

\vskip 0.3in
]

\printAffiliationsAndNotice{}  

\begin{abstract}
      In the recent literature the important role of depth in deep learning has been emphasized. In this paper
      we argue that sufficient width of a feedforward network is equally important by answering the simple question
      under which conditions the decision regions of a neural network are connected. It turns out that for a class of
      activation functions including leaky ReLU, neural networks having a pyramidal structure, that is no layer has
      more hidden units than the input dimension, produce necessarily connected decision regions. This implies that 
      a sufficiently wide hidden layer is necessary to guarantee that the network can produce disconnected decision regions. 
      We discuss the implications of this result for the construction of neural networks, 
      in particular the relation to the problem of adversarial manipulation of classifiers.
\end{abstract}


%

\section{Introduction}
While deep learning has become state of the art in many application domains such as computer vision and
natural language processing and speech recognition, the theoretical understanding of this success is steadily
growing but there are still plenty of questions where there is little or no understanding. In particular, for the question
how one should construct the network e.g. choice of activation function, number of layers, number of hidden units per layer etc.,
there is little guidance and only limited understanding on the implications of the choice e.g. ``The design of hidden units is an extremely active area of research and does not
yet have many definitive guiding theoretical principles.'' is a quote from the recent book on deep learning \citep[p. 191]{Goodfellow-et-al-2016}.
Nevertheless there is recently progress in the understanding of these choices. 

The first important results are the universal approximation theorems \citep{Cybenko1989,Hornik1989} 
which show that even a single hidden layer network with standard non-polynomial activation function \cite{Leshno1993}, like the sigmoid,
can approximate arbitrarily well every continuous function over a compact domain of $\RR^d$.
In order to explain the success of deep learning, much of the recent effort has been spent on analyzing the representation power of neural networks from the 
perspective of depth
\citep{Delalleau2011,Telgarsky2016,Telgarsky2015,Eldan2016,Safran2017,Yarotsky2016,Poggio2016,Liang2017,Mhaskar2016}.
Basically, they show that there exist functions that can be computed efficiently 
by deep networks of linear or polynomial size but require exponential size for shallow networks.
To further highlight the power of depth, \cite{Montufar2014, Pascanu2014} show that
the number of linear regions that a ReLU network can form in the input space grows exponentially with depth.
Tighter bounds on the number of linear regions are later on developed by \cite{AroraEtal2018, SerraEtal2018, ChaMar2018}.
Another measure of expressivity so-called trajectory length is proposed by \cite{Raghu2017}.
They show that the complexity of functions computed by the network 
along a one-dimensional curve in the input space also grows exponentially with depth.

While most of previous work 
can only show the existence of depth efficiency (\ie there exist certain functions that can be efficiently represented by
deep networks but not effectively represented or even approximated by shallow networks)
but cannot show how often this holds for all functions of interest,
\cite{CohenCOLT2016} have taken the first step to address this problem.
In particular, by studying a special type of networks called convolutional arithmetic circuits 
-- also known as Sum-Product networks \citep{Poon2011},
the authors show that besides a set of measure zero, 
all functions that can be realized by a deep network of polynomial size 
require exponential size in order to be realized, or even approximated by a shallow network.
Later, \cite{CohenICML2016} show that this property however no longer holds for convolutional rectifier networks,
which represents so far the empirically most successful deep learning architecture in practice.

Unlike most of previous work which focuses on the power of depth,
\citep{Lu2017,Hanin2017} have recently shown that neural networks with ReLU activation function
have to be wide enough in order to have the universal approximation property as depth increases.
In particular, the authors show that the class of continuous functions on a compact set cannot be arbitrarily well approximated by 
an arbitrarily deep network if the maximum width of the network is not larger than the input dimension $d$.
Moreover, it has been shown recently, that the loss surface of fully connected networks \citep{Quynh2017}
and for convolutional neural networks \citep{Quynh2017_conv} is well behaved, in the sense that almost all
local minima are global minima, if there exists a layer which has more hidden units than the number of training points.

In this paper we study the question under which conditions on the network the decision regions of a neural network are 
connected respectively can potentially be disconnected. The decision region of a class is the subset of $\RR^d$, 
where the network predicts this class. 
A similar study has been in \cite{Makhoul1989,Makhoul1990} for feedforward
networks with threshold activation functions, where they show that the initial layer has to have width $d+1$ in order that one can
get disconnected decision regions. On an empirical level it has recently been argued \cite{Fawzi2017} that the decision regions of 
the Caffe Network \cite{JiaEtAl2014} on ImageNet are connected. In this paper we analyze feedforward networks with 
continuous activation functions as currently used in practice. We show in line with previous work that almost all networks which
have a pyramidal structure up to the last hidden layer, that is the width of all hidden layers is smaller than the input dimension $d$,
can only produce connected decision regions. 
We show that the result is tight by providing explicit counterexamples for the case $d+1$. 
We conclude that a guiding principle
for the construction of neural networks should be that there is a layer which is wider than the input dimension 
as it would be a strong assumption that the Bayes optimal classifier must have connected decision regions. 
Interestingly, our result holds for leaky ReLU, that is $\sigma(t)=\max\{t,\alpha t\}$ for $0<\alpha<1$, 
whereas the result of \cite{Hanin2017} is for ReLU, that is $\sigma(t)=\max\{t,0\}$,
but ``the generalization is not straightforward, even for activations of the form $\sigma(t)=\max\{l_1(t),l_2(t)\}$, 
where $l_1,l_2$ are affine functions with different slopes.''
We discuss also the implications of connected decision regions regarding the generation of adversarial samples, 
which will provide another argument in favor of larger width for neural network architectures.

    
\section{Feedforward Neural Networks}\label{sec:fnn}
We consider in this paper feedforward neural networks for multi-class classification.
Let $d$ be the input dimension and $m$ the number of classes.
Let $L$ be the number of layers where the layers are indexed from $k=0,1,\ldots,L$
which respectively corresponds to the input layer, 1st hidden layer, $\ldots$, and the output layer $L$.
Let $n_k$ be the width of layer $k$. 
For consistency, we assume that $n_0=d$ and $n_L=m$.
Let $\sigma_k:\RR\to\RR$ be the activation function of every hidden layer $1\leq k\leq L-1$.
In the following, all functions are applied componentwise.
We define $f_k:\RR^d\to\RR^{n_k}$ as the feature map of layer $k$, which computes for every input $x\in\RR^d$ 
a feature vector at layer $k$ defined as
\begin{align*}
    f_k(x) = 
    \begin{cases}
	x & k=0\\
	\sigma_k\big( W_k^T f_{k-1}(x) + b_k \big) & 1\leq k\leq L-1\\
	 W_L^T f_{L-1}(x) + b_L & k=L
    \end{cases}
\end{align*}
where $W_k\in\RR^{n_{k-1}\times n_k}$ is the weight matrix at layer $k$.
Please note that the output layer is linear as it is usually done in practice. We consider in the following
activation functions $\sigma: \RR \rightarrow \RR$ which are continuous and strictly monotonically increasing.
This is true for most of proposed activation functions, but does not hold for ReLU, $\sigma(t)=\max\{t,0\}$. On the
other hand, it has been argued in the recent literature, that the following variants are to be preferred over ReLU
as they deal better with the vanishing gradient problem and outperform ReLU in prediction performance \cite{HeEtAl2015,Clevert2016}. This is leaky ReLU \cite{Maas2013}: 
        \[ \sigma(t)=\max\{t,\alpha t\} \quad \textrm{ for } \quad 0<\alpha<1,\] 
        where typically $\alpha$ is fixed but it has also been optimized
        together with the network weights \cite{HeEtAl2015} and ELU (exponential linear unit) \citep{Clevert2016}: 
        \[ \sigma(t) = \begin{cases}
 	e^t-1 & t < 0\\
 	t & t\geq 0 .
     \end{cases}.\]
Note that image of the activation function $\sigma$, $\sigma(\RR)=\{ \sigma(t)\,|\, t \in \RR\}$, is equal to $\RR$ for leaky ReLU and $(-1,\infty)$  for the
exponential linear unit.

\section{Connectivity of  Decision Regions}\label{sec:topology}
In this section, we prove two results on the connectivity of the decision
regions of a classifier. Both require that the activation function is continuous and strictly monotonically increasing.
Our main Theorem \ref{theo:class_region}
holds for feedforward networks of arbitrary depth and requires additionally $\sigma(\RR)=\RR$, 
the second Theorem \ref{theo:class_region_one} holds just for one hidden layer networks but has no further requirements on the activation function. 
Both show that in general pyramidal feedforward neural networks 
where the width of all the hidden layers is smaller than or equal to the input dimension
can only produce connected decision regions.

\subsection{Preliminary technical results}
We first introduce the definitions and terminologies used in the following, before we prove
or recall some simple results about continuous mappings from $\RR^m$ to $\RR^n$.
For a function $f:U\to V$, where $\dom(f)=U\subseteq\RR^m$ and $V\subseteq\RR^n$, we denote
for every subset $A\subseteq U$, the image $f(A)$ as   $f(A)\bydef\Setbar{f(x)}{x\in A}=\bigcup_{x\in A}f(x).$
Let $\range(f)\bydef f(U).$
\begin{definition}[Decision region]\label{def:class_region}
    The decision region of a given class $1\leq j\leq m$, denoted by $C_j$, is defined as
    \begin{align*}
	C_j = \Setbar{x\in\RR^d}{(f_L)_j(x) > (f_L)_k(x),\; \forall k\neq j} .
    \end{align*}
\end{definition}
\begin{definition}[Connected set]\label{def:connected_set}
    A subset $S\subseteq\RR^d$ is called connected if for every $x,y\in S$, 
    there exists a continuous curve $r: [0,1]\to S$ such that $r(0)=x$ and $r(1)=y.$
\end{definition}
To prove our key Lemma \ref{lem:inverse_layer}, 
the following properties of connected sets and continuous functions are useful.
All proofs are moved to the appendix due to limited space.
\begin{proposition}\label{prop:connected_continuous_map}
    Let $f:U\to V$ be a continuous function.
    If $A\subseteq U$ is a connected set then $f(A)\subseteq V$ is also a connected set.
\end{proposition}

\begin{proposition}\label{prop:minkowski}
    The Minkowski sum of two connected subsets $U,V\subseteq\RR^n$,
    defined as $U+V=\Setbar{u+v}{u\in U,v\in V}$, is a connected set.
\end{proposition}
As our main idea is to transfer the connectedness of a set from the output layer 
back to the input layer, we require the notion of pre-image and inverse mapping.
\begin{definition}[Pre-Image]
    The pre-image of a function $f:U\to V$ is the set-valued function $f^{-1}:V\to U$ defined 
    for every $y\in V$ as 
     \[ f^{-1}(y) = \Setbar{x\in U}{f(x)=y}.\] 
    Similarly, for every subset $A\subseteq V$, let 
\[ f^{-1}(A)=\bigcup_{y\in A}f^{-1}(y)=\Setbar{x\in U}{f(x)\in A}.\]
\end{definition}
By definition, it holds for every subset $A\subseteq V$ that $f(x)\in A$ if and only if $x\in f^{-1}(A).$
Moreover, for every $A\subseteq V$ 
\begin{align*}
    f^{-1}(A)
    &=f^{-1}(A\cap\range(f))\;\cup\;f^{-1}(A\setminus\range(f)) \\
    &=f^{-1}(A\cap\range(f))\;\cup\;\emptyset \\
    &=f^{-1}(A\cap\range(f)) .
\end{align*}
As a deep feedforward network is a composition of the individual layer functions, we need the following property.
\begin{proposition}\label{prop:inverse_compo}
    Let $f:U\to V$ and $g:V\to Q$ be two functions.
    Then it holds that $(g\circ f)^{-1} = f^{-1}\circ g^{-1}.$
\end{proposition}

\begin{figure*}[ht!]
\centering
    \begin{tikzpicture}[scale=1.8]
      \begin{scope}
	\draw[thick,->] (-1.5,0) -- (1.5,0) node[anchor=north west] {x};
	\draw[thick,->] (0,-1.5) -- (0,1.5) node[anchor=south east] {y};
	\draw (0.125cm,1pt) -- (0.125 cm,-1pt) node[anchor=north] {$\frac{1}{8}$};
	\draw (1pt,0.125 cm) -- (-1pt,0.125 cm) node[anchor=east] {$\frac{1}{8}$};
	\draw (1pt,1 cm) -- (-1pt,1 cm) node[anchor=east] {$1$};
	\draw (1pt,-1 cm) -- (-1pt,-1 cm) node[anchor=east] {$-1$};
	\draw (-1 cm,-1pt) -- (-1 cm,1pt) node[anchor=north] {$-1$};
	\draw (1 cm,-1pt) -- (1 cm,1pt) node[anchor=north] {$1$};
	\node at (0.5,0.5) {$S$};
	\coordinate (a) at (1.35,0.5);
	
	\draw[dashed] (-1.3,1.3)node[anchor=south west] {$f(\RR)$} -- (1.3,-1.3);
	\draw[thick,-,blue] (-1,1) -- (-0.125,0.125);
	\draw[thick,-,red] (0.125,-0.125) -- (1,-1);
	\fill[pattern=north east lines, pattern color=gray] (-1,0.125) -- (0.125,0.125) -- (0.125,-1) -- (1,-1) -- (1,1) -- (-1,1) -- cycle;
      \end{scope}
      
      \begin{scope}[xshift=3.5cm]
	\draw[thick,->] (-1.5,0) -- (1.5,0) node[anchor=north west] {};
	\draw (-1 cm,-1pt) -- (-1 cm,1pt) node[anchor=north] {$-1$};
	\draw (1 cm,-1pt) -- (1 cm,1pt) node[anchor=north] {$1$};
	\draw (0 cm,-1pt) -- (0 cm,1pt) node[anchor=north] {$0$};
	\draw (-0.5 cm,-1pt) -- (-0.5 cm,1pt) node[anchor=north] {$-\frac{1}{2}$};
	\draw (0.5 cm,-1pt) -- (0.5 cm,1pt) node[anchor=north] {$\frac{1}{2}$};
	\coordinate (b) at (-1.35,0.5);
	
	\fill[pattern=north east lines, pattern color=red] (-1,-0.05) -- (-0.5,-0.05) -- (-0.5,0.05) -- (-1,0.05) -- cycle;
	\fill[pattern=north east lines, pattern color=blue] (1,-0.05) -- (0.5,-0.05) -- (0.5,0.05) -- (1,0.05) -- cycle;
      \end{scope}
      
      \draw[->] (a) -- node [midway,above]{$f^{-1}(S)$} (b);
    \end{tikzpicture}
\caption{Left: illustration of the image of $\RR$ under the mapping $f$, denoted as $f(\RR) \subset \RR^2$ 
for the toy example from  \eqref{eq:toyex} which maps into a lower-dimensional subspace (the diagonal line). 
Right: The pre-image $f^{-1}(S) \subset \RR$ of the connected $S$ becomes disconnected.}
\label{fig:disc}
\end{figure*}
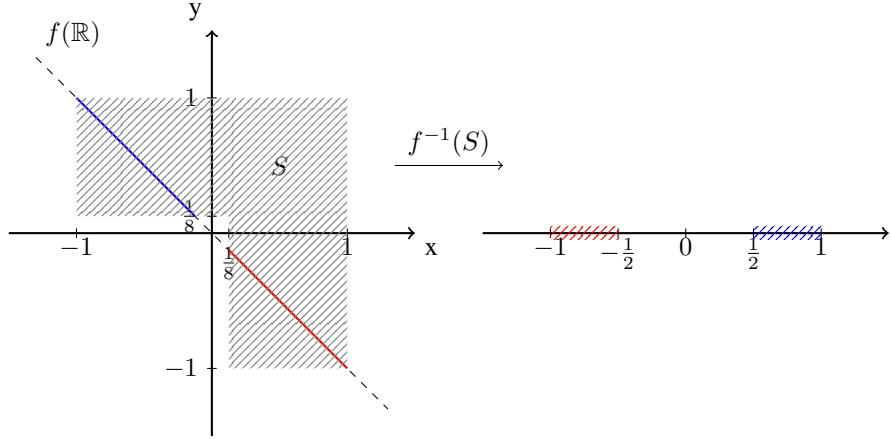

Apart from the property of connectivity, we can also show the openness of a set when considering the pre-image
of a given network.
We recall the following standard result from topology (see \eg \citealp{Apostol1974}, Theorem 4.23, p. 82).
\begin{proposition}\label{prop:inverse_open}
    Let $f:\RR^m\to\RR^n$ be a continuous function.
    If $U\subseteq\RR^n$ is an open set then $f^{-1}(U)$ is also open.
\end{proposition}
We now recall a standard result from calculus showing that under certain, restricted conditions the inverse
of a continuous mapping exists and is as well continuous.
\begin{proposition}\label{lem:inverseContFct}
Let $f:\RR \rightarrow f(\RR)$ be continuous and strictly monotonically increasing.
Then the inverse mapping $f^{-1}: f(\RR) \rightarrow \RR$ exists and is continuous.
\end{proposition}
The following lemma is a key ingredient in the following proofs. It allows us to show that the pre-image 
of an open and connected set by a one hidden layer network is again open and connected. Using the fact
that deep networks can be seen as a composition of such individual layers, this will later on allow us to transfer
the result to deep networks.
\begin{lemma}\label{lem:inverse_layer}
    Let $m\geq n$ and $f:\RR^m\to\RR^n$ be a function defined as 
    $f=\hat{\sigma} \circ h$ where $\hat{\sigma}:\RR^n \to\RR^n$ is  defined as 
    \begin{equation}
      \hat{\sigma}(x)=\begin{pmatrix} \sigma(x_1)\\ \vdots \\ \sigma(x_n)\end{pmatrix} \label{eq:hatsigma},
    \end{equation}
   and $\sigma:\RR \rightarrow \RR$ is bijective, continuous and strictly monotonically increasing,
    $h:\RR^m\to\RR^n$ is a linear map defined as 
    $h(x)=W^Tx+b$ where $W\in\RR^{m\times n}$ has full rank and  $b\in\RR^n.$
    If $V\subseteq\RR^n$ is an open connected set then 
    $f^{-1}(V)\subseteq\RR^m$ is also an open connected set.
\end{lemma}
\begin{proof}
    By Proposition \ref{prop:inverse_compo}, it holds that $f^{-1}(V)=h^{-1}(\hat{\sigma}^{-1}(V)).$
    As $\hat{\sigma}$ is  a componentwise function, the inverse mapping $\hat{\sigma}^{-1}$ is  given by the inverse mappings of the components
    \[ \hat{\sigma}^{-1}:\RR^n \rightarrow \RR^n, \; \hat{\sigma}^{-1}(x)=\begin{pmatrix} \sigma^{-1}(x_1)\\ \vdots \\ \sigma^{-1}(x_n)\end{pmatrix},\]  
    where under the stated assumptions the inverse mapping $\sigma^{-1}:\RR \rightarrow \RR$ exists by Lemma \ref{lem:inverseContFct} and is continuous.
    Since $V\subseteq \RR^n=\dom(\hat{\sigma}^{-1})$,  $\hat{\sigma}^{-1}(V)$  
    is the image of the connected set $V$ under the continuous map $\hat{\sigma}^{-1}$.
    Thus by Proposition \ref{prop:connected_continuous_map}, $\hat{\sigma}^{-1}(V)$ is connected.
    Moreover, $\hat{\sigma}^{-1}(V)$ is an open set by Proposition \ref{prop:inverse_open}.
    
    It holds for every $y\in\RR^n$ that
    \begin{align*}
	&h^{-1}(y) \\
	&= \begin{cases}
	    \emptyset & y\notin \range(h)\\
	    W(W^TW)^{-1} (y-b) + \ker(W^T) & y\in\range(h) ,
	\end{cases}
    \end{align*}
    where the inverse of $W^TW$ exists as $W$ has full rank $n$ (note that we assume $n\leq m$).
    As $W$ has full rank and $m\geq n$, it holds that $\range(h)=\RR^n$ and thus
    \begin{align*}
	h^{-1}(y) = W(W^TW)^{-1} (y-b) + \ker(W^T), \quad \forall\,y\in\RR^n.
    \end{align*}
    Therefore it holds for $\hat{\sigma}^{-1}(V)\subseteq\RR^n$ that
    \begin{align*}
	h^{-1}\big(\hat{\sigma}^{-1}(V)\big) 
	= W(W^TW)^{-1} \big(\hat{\sigma}^{-1}(V)-b\big) + \ker(W^T) ,
    \end{align*}
    where the first term is the image of the connected set $\hat{\sigma}^{-1}(V)$ 
    under an affine mapping and thus is again connected by Proposition \ref{prop:connected_continuous_map},
    the second term $\ker(W^T)$ is a linear subspace which is also connected.
    By Proposition \ref{prop:minkowski}, the Minkowski sum of two connected sets is connected.
    Thus $f^{-1}(V)=h^{-1}(\hat{\sigma}^{-1}(V))$ is a connected set.
    Moreover, as $f^{-1}(V)$ is the pre-image of the open set $V$ under the continuous function $f$,
    it must be also an open set by Proposition \ref{prop:inverse_open}.
    Thus $f^{-1}(V)$ is an open and connected set.
\end{proof}
Note that in Lemma \ref{lem:inverse_layer}, if $m<n$ and $W$ has full rank then $\range(h)\subsetneq\RR^n$ and the linear equation $h(x)=y$ 
has a unique solution $x=(WW^T)^{-1}W(y-b)$ for every $y\in\range(h)$ and thus
\begin{align*}
    f^{-1}(V)
    &=h^{-1}\big(\sigma^{-1}(V)\big) \\
    &=h^{-1}\big(\sigma^{-1}(V)\cap\range(h)\big) \\
    &= (WW^T)^{-1}W \big((\sigma^{-1}(V)\cap\range(h))-b\big) .
\end{align*}
In this case, even though $\sigma^{-1}(V)$ is a connected set, 
the intersection $\sigma^{-1}(V)\cap\range(h)$ can be disconnected
which can imply that $f^{-1}(V)$ is disconnected and thus the decision region becomes disconnected.

We illustrate this with a simple example, where $m=1$ and $n=2$ with $\sigma(x)=x^3$ and $W^T=\begin{pmatrix} -1 \\ 1\end{pmatrix}$ and $b=\begin{pmatrix} 0\\0\end{pmatrix}$. 
In this case it holds that
\begin{equation}\label{eq:toyex} 
f(x) = \hat{\sigma}(W^Tx+b)= \begin{pmatrix} \sigma(-x)\\ \sigma(x)\end{pmatrix} = \begin{pmatrix} -x^3\\ x^3 \end{pmatrix}.
\end{equation}
Figure \ref{fig:disc} shows that $f(\RR)$ is a one-dimensional submanifold (in this case subspace) of $\RR^2$
and provides an example of a set $S \subset \RR^2$ where the pre-image $f^{-1}(S)$ is disconnected.

\subsection{Main results}
We show in the following that the decisions regions of feedforward networks which are pyramidal 
and have maximal width at most the input dimension $d$ can only produce connected decision regions. 
We assume for the activation functions that $\sigma(\RR)=\RR$, which is fulfilled by leaky ReLU.
\begin{theorem}\label{theo:class_region} 
    Let the width of the layers of the feedforward network network satisfy $d=n_0\geq n_1\geq\ldots\geq n_{L-1}$
    and let $\sigma_l:\RR\to\RR$ be a continuous, strictly monotonically increasing function with $\sigma_l(\RR)=\RR$ for every layer $1\leq l\leq L-1$
    and all the weight matrices $(W_l)_{l=1}^{L-1}$ have full rank.
    Then every decision region $C_j$ is an open connected subset of $\RR^d$
    for every $1\leq j\leq m.$
\end{theorem}
\begin{proof}
    From Definition \ref{def:class_region}, it holds for every $1\leq j\leq m$ 
    \begin{align*}
    C_j \hspace{-1mm}&=\hspace{-1mm} \Setbar{x\in\RR^d}{f_{Lj}(x) - f_{Lk}(x) > 0,\forall k\neq j} 
    \end{align*} 
    where
    \begin{align*}
	&f_{Lj}(x)-f_{Lk}(x)\\
	&=\inner{(W_L)_{:j}-(W_L)_{:k}, f_{L-1}(x)} + {(b_L)}_j-{(b_L)}_k.
    \end{align*}
    Let us define the set
    $$V_j\!\!=\!\!\Setbar{y}{\inner{(W_L)_{:j}-(W_L)_{:k},y}\!>\!{(b_L)}_k-{(b_L)}_j,\forall k\neq j}$$
    then it holds
    $ C_j 
	= \Setbar{x\in\RR^d}{f_{L-1}(x)\in V_j} 
	= f_{L-1}^{-1} (V_j).
    $
    If $V_j$ is an empty set then we are done, 
    otherwise one observes that $V_j$ is the intersection of a finite number of open half-spaces (or the whole space), 
    which is thus an open and connected set. 
    Moreover, it holds $V_j \cap \hat{\sigma}_{L-1}(\RR)=V_j$, where $\hat{\sigma}_{L-1}$ is defined as in \eqref{eq:hatsigma}.
    It follows from Proposition \ref{prop:inverse_open} that $C_j$ must be an open set 
    as it is the pre-image of the open set $V_j$ under the continuous mapping $f_{L-1}.$
    To show that $C_j$ is a connected set, one first observes that
    \begin{align*}
	f_{L-1} = \hat{\sigma}_{L-1} \circ h_{L-1} \circ \hat{\sigma}_{L-2} \circ h_{L-2} \ldots \circ \hat{\sigma}_1 \circ h_1
    \end{align*}
    where $h_k:\RR^{n_{k-1}}\times\RR^{n_k}$ 
    is an affine mapping between layer $k-1$ and layer $k$ defined as $h_k(x)=W_k^T x + b_k$ for every $1\leq k\leq L-1, x\in\RR^{n_{k-1}}$,
    and $\hat{\sigma}_k:\RR^{n_k}\to\RR^{n_k}$ is the activation mapping of layer $k$ defined as in \eqref{eq:hatsigma}.
    By Proposition \ref{prop:inverse_compo} it holds that
    \begin{align*}
	f_{L-1}^{-1}(V_j) = (h_1^{-1} \circ \hat{\sigma}_1^{-1} \circ \ldots \circ h_{L-1}^{-1} \circ \hat{\sigma}_{L-1}^{-1})(V_j)
    \end{align*}
    Since $\sigma_k:\RR\to\RR$ is a continuous bijection by our assumption, 
    it follows that $\hat{\sigma}_k:\RR^{n_k}\to\RR^{n_k}$ is also a continuous bijection.
    Moreover, it holds that $W_k$ has full rank and $n_{k-1}\geq n_k$ for every $1\leq k\leq L-1$ and $V_j$ is a connected set.
    Thus one can apply Lemma \ref{lem:inverse_layer} subsequently for the composed functions $(\hat{\sigma}_k\circ h_k)$
    for every $k=L-1,L-2,\ldots,1$ and obtains that $C_j=f_{L-1}^{-1}(V_j)$ is a connected set.
    Thus $C_j$ is an open and connected set for every $1\leq j\leq m.$
\end{proof}
The next theorem holds just for networks with one hidden layer but allows general activation functions which are continuous
and strictly monotonically increasing, that is leaky ReLU, ELU, softplus or sigmoid activation functions. Again the decision
regions are connected if the hidden layer has maximal width smaller than $d+1$.
\begin{theorem}\label{theo:class_region_one} 
    Let the one hidden layer network satisfy $d=n_0 \geq n_1$
    and let $\sigma_1:\RR\to\RR$ be a continuous, strictly monotonically increasing function
    and the hidden layer's weight matrix $W_1$ has full rank.
    Then every decision region $C_j$ is an open connected subset of $\RR^d$ for every $1\leq j\leq m.$
\end{theorem}
\begin{proof}
We note that in the proof of Theorem \ref{theo:class_region} the  $V_j$ is a finite intersection of open half-spaces and thus a convex set.
Moreover, $\hat{\sigma}_1(\RR^{n_1})$ is an open convex set (it is just an axis-aligned open box), as $\sigma_1$ is strictly
monotonically increasing. Thus
\begin{align*}
	C_j 
	&= \Setbar{x\in\RR^d}{f_{1}(x)\in V_j \cap \hat{\sigma}_1(\RR^{n_1})} \\
	&= f_{1}^{-1} \big(V_j \cap \hat{\sigma}_1(\RR^{n_1})\big).
    \end{align*}
As both sets are open convex sets, the intersection $V_j \cap \hat{\sigma}_1(\RR^{n_1})$ is again convex and open as well. 
Thus $V_j \cap \hat{\sigma}_1(\RR^{n_1})$ is a connected set. The rest of the argument follows then by using Lemma \ref{lem:inverse_layer},
noting that by Proposition \ref{lem:inverseContFct} ${\hat{\sigma}_1}^{-1}: \hat{\sigma}_1(\RR^{n_1}) \rightarrow \RR^{n_1}$ is a continuous mapping.
\end{proof}
Note that Theorem \ref{theo:class_region} and Theorem \ref{theo:class_region_one} make no assumption on the structure of all layers in the network.
Thus they can be applied to neural networks with both fully connected layers and convolutional layers.
Moreover, the results hold regardless of how the parameters of the network $(W_l,b_l)_{l=1}^L$ have been attained, 
trained or otherwise, as long as all the weight matrices of hidden layers have full rank.
This is a quite weak condition in practice as the set of low rank matrices has just Lebesgue measure zero. Even if the optimal weight parameters
for the data generating distribution would be low rank (we discuss such an example below), then it is very unlikely that the trained weight parameters are low rank, as one has
statistical noise by the training sample, ``optimization noise'' from the usage of stochastic gradient descent (SGD)
and its variants and finally
in practice one often uses early stopping and thus even if the optimal solution for the training set is low rank,
one will not find it.

Theorem  \ref{theo:class_region} covers activation functions like leaky  ReLU but not sigmoid, ELU or softplus. At the moment it is unclear
for us if the result might hold also for the more general class of activation functions treated in Theorem \ref{theo:class_region_one}. 
The problem is that then in Lemma  \ref{lem:inverse_layer} one has to compute the pre-image of $V \cap \hat{\sigma}(\RR^n)$. Even
though both sets are connected, the intersection of connected sets need not be connected. This is avoided in Theorem \ref{theo:class_region_one}
by using that the initial set $V_j$ and $\hat{\sigma}(\RR^{n_{L-1}})$ 
are both convex and the intersection of convex sets is convex and thus connected.

We show below that the result is tight in the sense that we give an empirical example of a neural network with a single hidden layer of $d+1$ hidden units
which produces disconnected regions. 
Note that our result complements the result of \cite{Hanin2017}, 
where they show the universal approximation property (for ReLU) only if
one considers networks of width at least $d+1$ for arbitrary depth. 
Theorem \ref{theo:class_region} and Theorem \ref{theo:class_region_one} indicate that this result could also hold for leaky ReLU as
approximation of arbitrary functions implies approximation of arbitrary decisions regions, which clearly requires that one is able to get disconnected decision regions.
Taking both results together, it seems rather obvious that as a general guiding principle for the construction of hidden layers in neural networks 
one should use, at least for the first hidden layer, more units than the input dimension, 
as it is rather unlikely that the Bayes optimal decision regions are connected. 
Indeed, if the true decision regions are disconnected
then using a network of smaller width than $d+1$ might still perfectly fit the finite training data 
but since the learned decision regions are connected there exists a path between the
true decision regions which then can be used for potential adversarial manipulation.
This is discussed in the next section where we show empirical evidence for the existence of such adversarial examples.

\section{Illustration and Discussion}
In this section we discuss with analytical examples as well as trained networks that the result is tight
and the conditions of the theorem cannot be further relaxed. Moreover, we argue that connected decision regions
can be problematic as they open up the possibility to generate adversarial examples.

\subsection{Why pyramidal structure of the network is necessary to get connected decision regions?}\label{sec:non_pyramidal}
\begin{figure}[ht]
\vspace{-5pt}
\centering
    \includegraphics[width=0.4\linewidth]{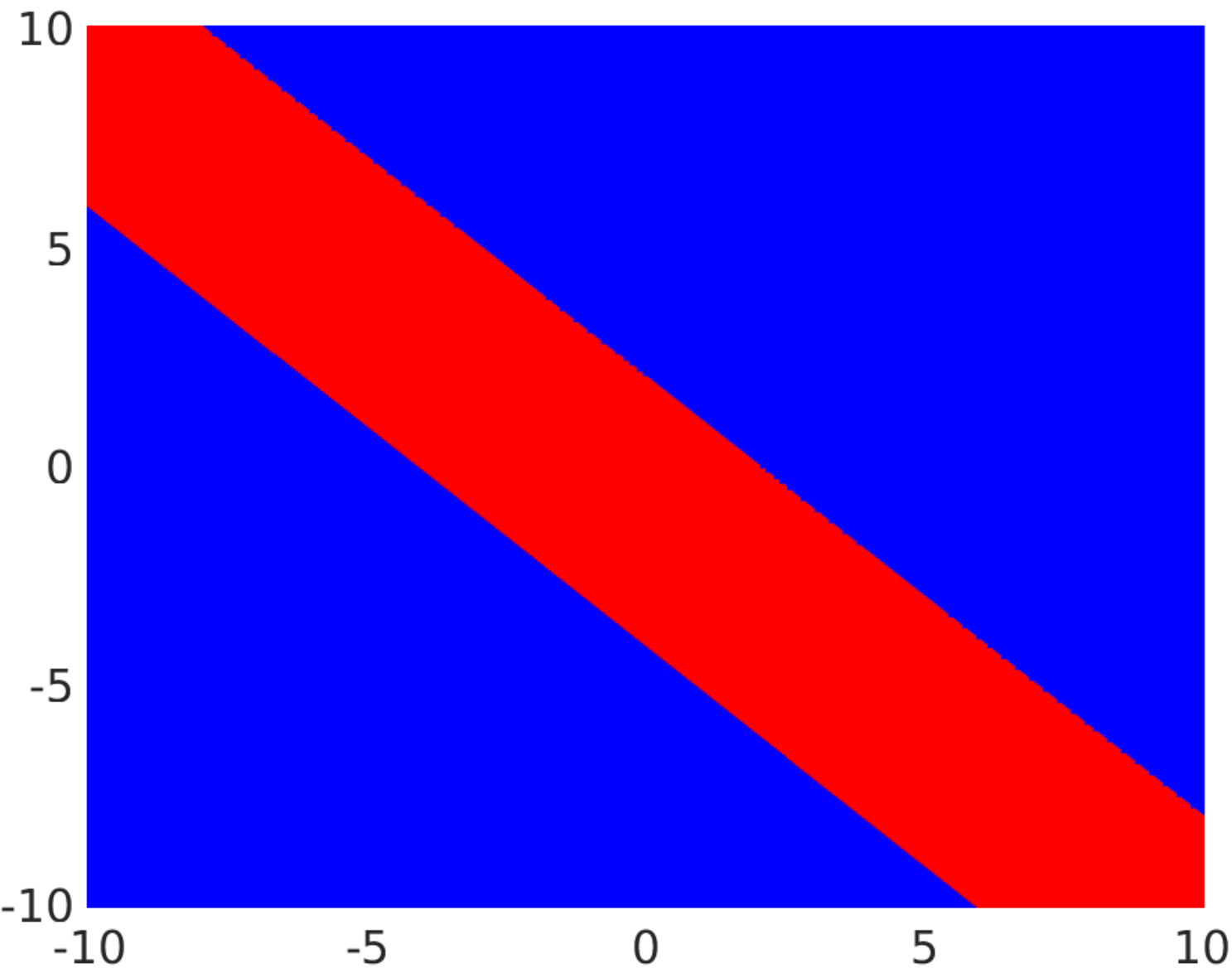}
    \hspace{5pt}
    \includegraphics[width=0.42\linewidth]{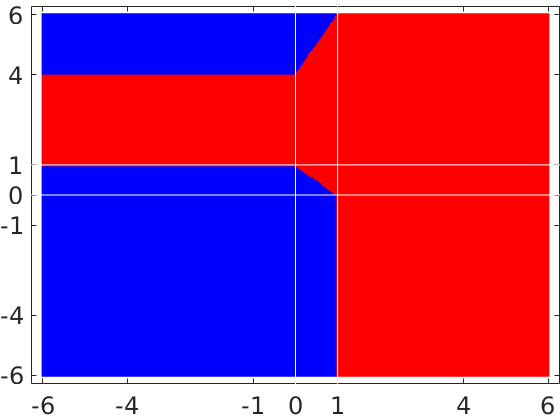}
\caption{Decision region of the network in \eqref{eq:non_pyramidal}(left) and \eqref{eq:net_relu}(right).}
\label{fig:relu}
\vspace{-5pt}
\end{figure}

%

In Theorem \ref{theo:class_region}, if the network does not have pyramidal structure up to the last hidden layer, 
\ie the condition $d_1\geq \ldots\geq d_{L-1}$ is not fulfilled, then the statement of the theorem
might not hold as the decision regions can be disconnected.
We illustrate this via a counter-example below.
Let us consider a non-pyramidal network 2-1-2-2 defined as
\begin{align}\label{eq:non_pyramidal}
    W_3^T\hat{\sigma}_2(W_2^T\hat{\sigma}_1(W_1^Tx+b_1)+b_2)+b_3
\end{align}
where $\sigma_1(t)=\sigma_2(t)=\max\Set{0.5\,t, t}$,
and
$		
    W_1=\begin{bmatrix}1\\1\end{bmatrix}, b_1=0, 
    W_2=\begin{bmatrix}1&-1\end{bmatrix}, b_2=\begin{bmatrix}0\\0\end{bmatrix}, 
    W_3=\begin{bmatrix}2&1\\3&2\end{bmatrix}, b_3=\begin{bmatrix}0\\1\end{bmatrix} .
$
Then one can check that this network has (see appendix for the full derivation)
$C_1 = \Setbar{x\in\RR^2}{x_1+x_2-2>0 \textrm{ and } x_1+x_2+4<0}$,
which is a disconnected set as illustrated in Figure \ref{fig:relu}.

\subsection{Why full rank of the weight matrices is necessary to get connected decision regions?}
Similar to Section \ref{sec:non_pyramidal}, we show that if the weight matrices of hidden layers are not full rank while the other conditions 
are still satisfied, then the decision regions can be disconnected. 
The reason is simply that low rank matrices, in particular in the first layer, reduce the effective
dimension of the input. We illustrate this effect with a small analytical example and then argue that nevertheless in 
practice it is extremely difficult to get low rank weight matrices.

\begin{figure}[ht]
\centering
    \includegraphics[width=0.4\linewidth]{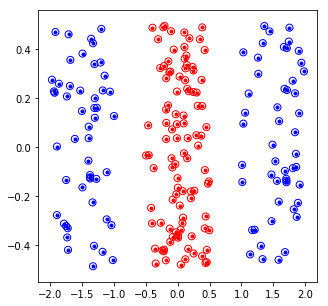}
    \quad 
    \includegraphics[width=0.4\linewidth]{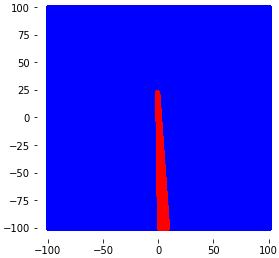}
\vspace{-5pt}
\caption{Left: the training set corresponding to the distribution in \eqref{onedproblem}. Right: the decision regions
of a trained classifier, which are connected as the learned weight matrix $W_1$ has full rank.}
\label{fig:illu1}
\vspace{-5pt}
\end{figure}
\vspace{-5pt}

\begin{figure*}[ht]
\begin{center}
    \subfigure[Training data]{
	\includegraphics[width=0.15\linewidth]{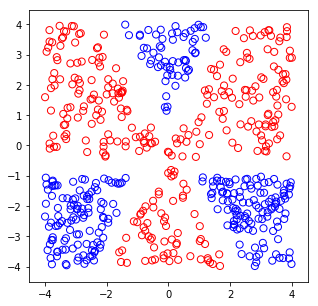}
    }
    \subfigure[$n_1=2 (122)$]{
      \includegraphics[width=0.15\textwidth]{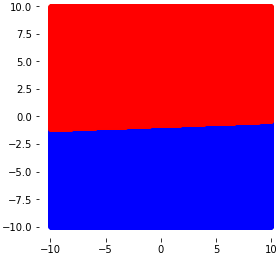}
    }
    \subfigure[$n_1=3 (72)$]{
      \includegraphics[width=0.15\textwidth]{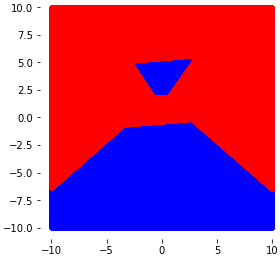}
    }
    \subfigure[$n_1=4 (45)$]{
      \includegraphics[width=0.14\textwidth]{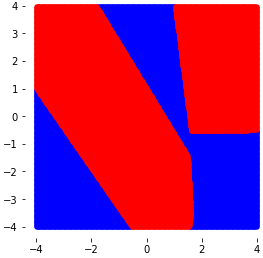}
    }
    \subfigure[$n_1=7 (12)$]{
      \includegraphics[width=0.15\textwidth]{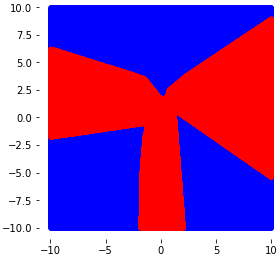}
    }
    \subfigure[$n_1=50 (0)$]{
      \includegraphics[width=0.15\textwidth]{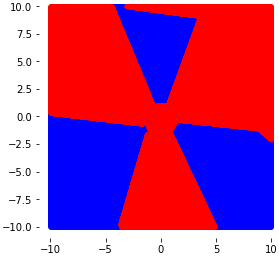}
    }
\end{center}
\vspace{-7pt}
\caption{Decision region of a one hidden layer network trained with SGD for varying number of hidden units
for the toy example given in (a). As shown by Theorem \ref{theo:class_region} the decision
region for $n_1=d=2$ is connected, however already for $n_1>d=2$ one gets disconnected decision regions which shows
that Theorem \ref{theo:class_region} is tight. The numbers in bracket show the number of misclassified training points.
}
\label{fig:toy2d}
\vspace{-5pt}
\end{figure*}

Suppose one has a two-class classification problem on $\RR^2$ (see Figure \ref{fig:illu1}) with equal class probabilities $P(\textrm{red})=P(\textrm{blue})$,
and the conditional distribution is given as
\begin{align}\label{onedproblem}
 p(x_1,x_2|\mathrm{blue})\hspace{-1mm} &=\hspace{-1mm} \frac{1}{2}, \,\forall \, x_1 \in [-2,-1] \cup [1,2], x_2 \in [-\frac{1}{2},\frac{1}{2}] \nonumber\\
 p(x_1,x_2|\mathrm{red}) \hspace{-1mm}&=\hspace{-1mm} 1, \,\forall\, x_1 \in [-1,1], x_2 \in [-\frac{1}{2},\frac{1}{2}].
\end{align}
Note that the Bayes optimal decision region for class blue is disconnected. Moreover, it is easy to verify that a one
hidden layer network with leaky ReLU $\sigma(t)=\max\{t,\alpha t\}$ for $0<\alpha<1$ can perfectly fit the data with 
\[ W^T_1 \hspace{-1mm}=\hspace{-1mm} \begin{pmatrix} 1 & 0\\-1 & 0 \end{pmatrix}, 
b_1\hspace{-1mm}=\hspace{-1mm}\begin{pmatrix}-1\\-1\end{pmatrix},
W^T_2 \hspace{-1mm}=\hspace{-1mm}\begin{pmatrix} 1 & 1 \\ 0 & 0\end{pmatrix}, 
b_2\hspace{-1mm}=\hspace{-1mm}\begin{pmatrix}0\\-2\alpha\end{pmatrix}\]
Note that $W_1$ has low rank.
Suppose that the first output unit corresponds to the blue class and second output unit corresponds to the red class.
Then it holds $(f_2)_{\mathrm{red}}(x_1,x_2) =-2\alpha$,
$(f_2)_{\mathrm{blue}}(x_1,x_2) = \max\{x_1-1,\alpha (x_1-1)\} + \max\{-(x_1+1), -\alpha (x_1+1)\}$
and thus 
\begin{align*}
(f_2)_{\mathrm{blue}}(x_1,x_2)\hspace{-1mm}=\hspace{-1mm}
\begin{cases} (1-\alpha)x_1 - (1+\alpha) & x_1\geq 1\\ 
-2\alpha & -1\leq x_1\leq 1\\
-(1-\alpha)x_1-(1+\alpha) & x_1\leq -1
\end{cases}
\end{align*}
which implies that $(f_2)_{\mathrm{blue}}(x_1,x_2)>(f_2)_{\mathrm{red}}(x_1,x_2)$ for every $x_1\in(-\infty,-1)\cup(1,+\infty)$
and thus the decision region for class blue has two disconnected decision regions. 
This implies that Theorems \ref{theo:class_region} and \ref{theo:class_region_one} do indeed not hold 
if the weight matrices do not have full rank. 
Nevertheless in practice, it is unlikely that one will get such low rank weight matrices, 
which we illustrate in Figure \ref{fig:illu1} that the decision regions of the trained classifier has indeed connected decision regions. 
This is due to statistical noise in the training set as
well as through the noise in the optimization procedure (SGD) and the common practice of early stopping in training of neural networks.

\subsection{Does the result hold for ReLU activation function?}
As the conditions of Theorem \ref{theo:class_region} are not fulfilled for ReLU, 
one might ask whether the decision regions of a ReLU network with pyramidal structure and full rank weight matrices
can be potentially disconnected.
We show that this is indeed possible via the following example.
Let a two hidden layer network (2-2-2-2) be defined as
\begin{align}\label{eq:net_relu}
    W_3^T\hat{\sigma}_2(W_2^T\hat{\sigma}_1(W_1^Tx+b_1)+b_2)+b_3
\end{align}
where  $\sigma_1(t)=\sigma_2(t)=\max\Set{t, 0}$ and
\begin{align*}		
    W_1^T=\begin{bmatrix}1&0\\0&1\end{bmatrix}, 
    W_2^T=\frac{\sqrt{2}}{2}\begin{bmatrix}1&1\\-1&1\end{bmatrix}, 
    W_3^T=\begin{bmatrix}-1&0\\ 0&-3\end{bmatrix}, 
\end{align*}
and $b_1=[0,0]^T, b_2=\frac{1}{\sqrt{2}}[\sqrt{2}-1, -3]^T, b_3=[1,0]^T.$
Then one can derive the decision region for the first class as (see appendix for the full derivation)
\begin{align*}
    C_1 
    &= \Setbar{x\in\RR^2} {x_1<1,\;x_2<1,\;x_1+x_2<1} 
    \\ 
    &\;\cup \Setbar{x\in\RR^2}{x_2>4,\;2x_1-x_2+4<0} 
\end{align*}
which is a disconnected set as illustrated in Figure \ref{fig:relu}.

Finally, one notes in this example that except for the activation function, 
all the other conditions of Theorem \ref{theo:class_region} are still satisfied, that is, 
the network has pyramidal structure (2-2-2-2) and all the weight matrices $(W_l)_{l=1}^{2}$ have full rank by our construction.
Thus the statement of Theorem \ref{theo:class_region}, at least under current form, does not hold for ReLU.


\subsection{The theorems are tight: disconnected decision regions for width $d+1$} 
We consider a binary classification task in $\RR^2$ where the data points are generated 
so that the blue class has disconnected components on the square $[-4,4]\times [-4,4]$, see Figure \ref{fig:toy2d} (a) for an illustration.
We use a one hidden layer network with varying number of hidden units,
two output units, leaky ReLU activation function and cross-entropy loss.
We then train this network by using SGD with momentum for $1000$ epochs and learning rate $0.1$ and reduce the it 
by a factor of $2$ after every $50$ epochs. 
For all the attempts with different starting points that we have done in our experiment, 
the resulting weight matrices always have full rank.

We show the training error and the decision regions of trained network in Figure \ref{fig:toy2d}.
The grid size in each case of Figure \ref{fig:toy2d} 
has been manually chosen so that one can see clearly the connected/disconnected components in the decision regions.
First, we observe that for two hidden units ($n_1=2$), the network satisfies the condition of Theorem \ref{theo:class_region}
and thus can only learn connected regions, which one can also clearly see in the figure, where one basically gets a linear separator.
However, for three hidden units ($n_1=3$), one can see that the network can produce disconnected decision regions,
which shows that both our Theorems \ref{theo:class_region} and \ref{theo:class_region_one} are tight, in the sense that width $d+1$
is already sufficient to produce disconnected components, whereas the results say that for width less than $d+1$ the decision regions
have to be connected.
As the number of hidden units increases, we observe  that the network produces more easily disconnected decision regions as expected.

\subsection{Relation to adversarial manipulation}\label{sec:adversarial}
We use a single image of digit $1$ from the MNIST dataset to create a new artificial dataset where the underlying data generation probability measure has a similar one-dimensional structure as in \eqref{onedproblem} but now embedded in the pixel space $\RR^{28 \times 28}$.
This is achieved by using rotation as the one-dimensional degree of freedom. 
We generate $2000$ training images for each red/blue class by rotating the chosen digit $1$ with angles ranging from $[-5^{\circ},5^{\circ}]$ 
for the read class, and $[-20^{\circ},-15^{\circ}]\,\cup\,[15^{\circ},20^{\circ}]$ for the blue class, see Figure \ref{fig:fig_classifications}. 
\begin{figure}[h]
\vspace{-3pt}
\centering  
\scriptsize
\stackunder{\label{fig:origin}\includegraphics[width=0.085\textwidth]{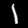}}{\textcolor{blue}{$[-20^{\circ},-15^{\circ}]$}}
\hspace{10pt}
\stackunder{\label{fig:main}\includegraphics[width=0.085\textwidth]{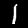}}{\textcolor{red}{$[-5^{\circ},5^{\circ}]$}}
\hspace{10pt}
\stackunder{\label{fig:target}\includegraphics[width=0.085\textwidth]{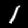}}{\textcolor{blue}{$[15^{\circ},20^{\circ}]$}}
\vspace{-3pt}
\caption{Training examples for our binary digit-$1$ dataset. 
The color (red/blue) denotes the class of corresponding example.
}
\label{fig:fig_classifications}
\vspace{-5pt}
\end{figure}

\begin{figure*}[t]
\centering 
\scriptsize
\stackunder{\includegraphics[width=0.085\textwidth]{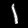}}{\textcolor{blue}{$0.99$ (source)}}
\stackunder{\includegraphics[width=0.085\textwidth]{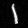}}{\textcolor{blue}{$0.99 (0.1)$}}
\stackunder{\includegraphics[width=0.085\textwidth]{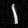}}{\textcolor{blue}{$0.99 (0.2)$}}
\stackunder{\includegraphics[width=0.085\textwidth]{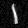}}{\textcolor{blue}{$0.99 (0.3)$}}
\stackunder{\includegraphics[width=0.085\textwidth]{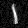}}{\textcolor{blue}{$0.99 (0.4)$}}
\stackunder{\includegraphics[width=0.085\textwidth]{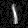}}{\textcolor{blue}{$0.99 (0.5)$ }}
\stackunder{\includegraphics[width=0.085\textwidth]{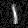}}{\textcolor{blue}{$0.99 (0.6)$ }}
\stackunder{\includegraphics[width=0.085\textwidth]{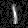}}{\textcolor{blue}{$0.99 (0.7)$}}
\stackunder{\includegraphics[width=0.085\textwidth]{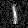}}{\textcolor{blue}{$0.99 (0.8)$}}
\stackunder{\includegraphics[width=0.085\textwidth]{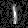}}{\textcolor{blue}{$0.99 (0.9)$}}
\stackunder{\includegraphics[width=0.085\textwidth]{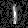}}{\textcolor{blue}{$0.99$ (\red{adversarial})}}
\vspace{5pt}
\stackunder{\includegraphics[width=0.085\textwidth]{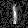}}{\textcolor{blue}{$0.99$ (\red{adversarial})}}
\stackunder{\includegraphics[width=0.085\textwidth]{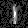}}{\textcolor{blue}{$0.99 (0.1)$}}
\stackunder{\includegraphics[width=0.085\textwidth]{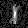}}{\textcolor{blue}{$0.99 (0.2)$}}
\stackunder{\includegraphics[width=0.085\textwidth]{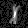}}{\textcolor{blue}{$0.99 (0.3)$}}
\stackunder{\includegraphics[width=0.085\textwidth]{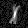}}{\textcolor{blue}{$0.99 (0.4)$}}
\stackunder{\includegraphics[width=0.085\textwidth]{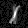}}{\textcolor{blue}{$0.99 (0.5)$ }}
\stackunder{\includegraphics[width=0.085\textwidth]{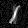}}{\textcolor{blue}{$0.99 (0.6)$ }}
\stackunder{\includegraphics[width=0.085\textwidth]{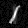}}{\textcolor{blue}{$0.99 (0.7)$}}
\stackunder{\includegraphics[width=0.085\textwidth]{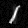}}{\textcolor{blue}{$0.99 (0.8)$}}
\stackunder{\includegraphics[width=0.085\textwidth]{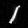}}{\textcolor{blue}{$0.99 (0.9)$}}
\stackunder{\includegraphics[width=0.085\textwidth]{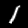}}{\textcolor{blue}{$0.99$ (target)}}
\vspace{-7pt}
\caption{ 
Digit-1 dataset ($2$ output classes): The trajectory from source image to adversarial image (top row) parameterized by $\lambda$ (numbers inside brackets), 
and from adversarial image to target image (second row). 
Each number outside bracket shows the confidence that the corresponding image was predicted to be in blue class.
The image with red caption can be seen as an adversarial image of the red class.
}
\label{fig:digit_1_trajectory_1}
\vspace{-5pt}
\end{figure*}

\begin{figure*}[t]
\centering 
\scriptsize
\stackunder{\includegraphics[width=0.085\textwidth]{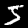}}{\textcolor{blue}{$0.99$ (source)}}
\stackunder{\includegraphics[width=0.085\textwidth]{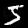}}{\textcolor{blue}{$0.99 (0.1)$}}
\stackunder{\includegraphics[width=0.085\textwidth]{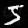}}{\textcolor{blue}{$0.98 (0.2)$}}
\stackunder{\includegraphics[width=0.085\textwidth]{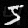}}{\textcolor{blue}{$0.98 (0.3)$}}
\stackunder{\includegraphics[width=0.085\textwidth]{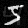}}{\textcolor{blue}{$0.97 (0.4)$}}
\stackunder{\includegraphics[width=0.085\textwidth]{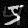}}{\textcolor{blue}{$0.96 (0.5)$ }}
\stackunder{\includegraphics[width=0.085\textwidth]{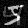}}{\textcolor{blue}{$0.94 (0.6)$ }}
\stackunder{\includegraphics[width=0.085\textwidth]{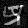}}{\textcolor{blue}{$0.93 (0.7)$}}
\stackunder{\includegraphics[width=0.085\textwidth]{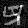}}{\textcolor{blue}{$0.92 (0.8)$}}
\stackunder{\includegraphics[width=0.085\textwidth]{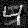}}{\textcolor{blue}{$0.89 (0.9)$}}
\stackunder{\includegraphics[width=0.085\textwidth]{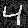}}{\textcolor{blue}{$0.82$ (\red{adversarial})}}
\vspace{5pt}
\stackunder{\includegraphics[width=0.085\textwidth]{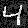}}{\textcolor{blue}{$0.82$ (\red{adversarial})}}
\stackunder{\includegraphics[width=0.085\textwidth]{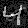}}{\textcolor{blue}{$0.55 (0.1)$}}
\stackunder{\includegraphics[width=0.085\textwidth]{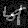}}{\textcolor{blue}{$0.45 (0.2)$}}
\stackunder{\includegraphics[width=0.085\textwidth]{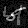}}{\textcolor{blue}{$0.54 (0.3)$}}
\stackunder{\includegraphics[width=0.085\textwidth]{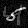}}{\textcolor{blue}{$0.68 (0.4)$}}
\stackunder{\includegraphics[width=0.085\textwidth]{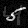}}{\textcolor{blue}{$0.83 (0.5)$ }}
\stackunder{\includegraphics[width=0.085\textwidth]{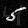}}{\textcolor{blue}{$0.91 (0.6)$ }}
\stackunder{\includegraphics[width=0.085\textwidth]{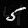}}{\textcolor{blue}{$0.96 (0.7)$}}
\stackunder{\includegraphics[width=0.085\textwidth]{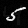}}{\textcolor{blue}{$0.98 (0.8)$}}
\stackunder{\includegraphics[width=0.085\textwidth]{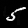}}{\textcolor{blue}{$0.99 (0.9)$}}
\stackunder{\includegraphics[width=0.085\textwidth]{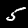}}{\textcolor{blue}{$0.99$ (target)}}
\vspace{-7pt}
\caption{ 
MNIST dataset ($10$ output classes): The trajectory from source image to adversarial image (top row) parameterized by $\lambda$ (numbers inside brackets), 
and from adversarial image to target image (second row). 
Each number outside bracket shows the confidence that the corresponding image was predicted to be in blue class (digit $5$) 
out of $10$ classes.
}
\label{fig:mnist_trajectory_1}
\end{figure*}
Note that this is a binary classification task where the dataset has just one effective degree of freedom 
and the Bayes optimal decision regions are disconnected. 
We train a one hidden layer network with $784$ hidden units
which is equal to the input dimension and leaky ReLU as activation function with $\alpha=0.1$. 
The training error is zero and the resulting weight matrices have full rank, 
thus the conditions of Theorem \ref{theo:class_region} are satisfied and the decision region of class blue should
be connected even though the Bayes optimal decision region is disconnected. 
This can only happen by establishing a connection around the other red class.
We test this by sampling a source image from the $[-20^{\circ},-15^{\circ}]$ part of the blue class 
and a target image from the other part $[15^{\circ},20^{\circ}]$.
Next, we generate an adversarial image 
\footnote{This is essentially a small perturbation of an image from the red class which is classified as blue class} 
from the red class using the one step target class method \cite{KurakinEtal2016, KurakinEtal2017} 
and consider the path between the source image to the adversarial image and subsequently from the adversarial image to the target one.
For each path, we simply consider the line segment $\lambda s + (1-\lambda)t$ for $\lambda \in [0,1]$ between
the two endpoint images $s$ and $t$ and sample it very densely by dividing $[0,1]$ into $10^4$ equidistant parts. 
Figure \ref{fig:digit_1_trajectory_1} shows the complete path from the source image to the target image 
where the color indicates that all the intermediate images are classified as blue with high confidence 
(note that we turned the output of the network into probabilities by using the softmax function). 
Moreover, the intermediate images from Figure \ref{fig:digit_1_trajectory_1} look very much like images from the red class 
thus could be seen as adversarial samples for the red class. 
The point we want to make here is that one might think that in order to 
avoid adversarial manipulation the solution is to use a simple classifier of low capacity. 
We think that rather the opposite is true in the sense that only if the classifier is rich enough to model the true underlying data generating distribution it will be able to model the true decision
boundaries. In particular, the classifier should be able to realize disconnected decision regions in order to avoid paths through the input space which connect different disconnected regions of the Bayes optimal classifier. Now one could argue that the problem of our synthetic example
is that the corresponding digits obviously do not fill the whole image space, nevertheless the classifier has to do a prediction for all possible images. This could be handled by introducing a background class, but then it would be even more important that the classifier can produce
disconnected decision regions which naturally requires a minimal width of $d+1$ of the network. 

In Figure \ref{fig:mnist_trajectory_1}, we show another similar experiment on MNIST dataset, but now for all the $10$ image classes.
We train a network with $200$ hidden units, leaky ReLU and softmax cross-entropy loss to zero training error.
Once again, one can see that there exists a continuous path that connects two different-looking images of digit $5$ (blue class)
where every image along this path is classified as blue class with high confidence.
Moreover this path goes through a pre-constructed adversarial image of the red class (digit $4$).

\section{Conclusion}
We have shown that deep neural networks (with a certain class of activation functions) need to have in general 
width larger than the input dimension in order to learn disconnected decision regions. It remains an open
problem if our current requirement $\sigma(\RR)=\RR$ can be removed. 
While our result does not resolve
the question how to choose the network architecture in practice, it provides at least a guideline how to choose the width
of the network. Moreover, our result and experiments show that too narrow networks produce high confidence predictions 
on a path connecting the true disconnected decision regions which could be used to attack these networks using adversarial
manipulation.

\section*{Acknowledgements}
The authors would like to thank the reviewers for their helpful comments on the paper
and Francesco Croce for bringing up a counter-example for the ReLU activation function.

\bibliography{regul}
\bibliographystyle{icml2018}

\ifpaper
\appendix
\textbf{Proof of Proposition 3.3}
    Pick some $a,b\in f(A).$
    Then there must exist some $x,y\in A$ such that $f(x)=a$ and $f(y)=b.$
    Since $A$ is a connected set, it holds by Definition 3.2 that
    there exists a continuous curve $r: [0,1]\to A$ such that $r(0)=x,r(1)=y.$  Consider the curve $f\circ r:[0,1]\to f(A)$,
    then it holds that $f(r(0))=a, f(r(1))=b.$  Moreover, $f\circ r$ is continuous as both $f$ and $r$ are continuous.
    Thus it holds that $f(A)$ is a connected set by Definition 3.2.

\paragraph{Proof of Proposition 3.4}
    Let $x,y\in U+V$, then there exists $a,b\in U$ and $c,d\in V$
    such that $x=a+c,y=b+d.$
    Since $U$ and $V$ are connected sets, there exist two continuous curves $p:[0,1]\to U$
    and $q:[0,1]\to V$ such that 
    $p(0)=a,p(1)=b$ and $q(0)=c,q(1)=d.$
    Consider the continuous curve $r(t)\bydef p(t)+q(t)$ then 
    it holds that $r(0)=a+c=x, r(1)=b+d=y$ and $r(t)\in U+V$ for every $t\in[0,1].$
    This implies that every two elements in $U+V$ can be connected by a continuous curve
    and thus $U+V$ must be a connected set.

\paragraph{Proof of Proposition 3.6}
    It holds for every $A\subseteq Q$ that
    \begin{align*}
	(g\circ f)^{-1}(A) 
	&= \Setbar{x\in U}{g(f(x))\in A} \\
	&= \Setbar{x\in U}{f(x)\in g^{-1}(A)} \\
	&= \Setbar{x\in U}{x\in f^{-1}(g^{-1}(A))} \\
	&= (f^{-1}\circ g^{-1})(A) .
    \end{align*}

\section*{Why pyramidal structure of the network is necessary to get connected decision regions?}
In the following, we show how to derive the decision regions of the network given in Equation \eqref{eq:non_pyramidal}.
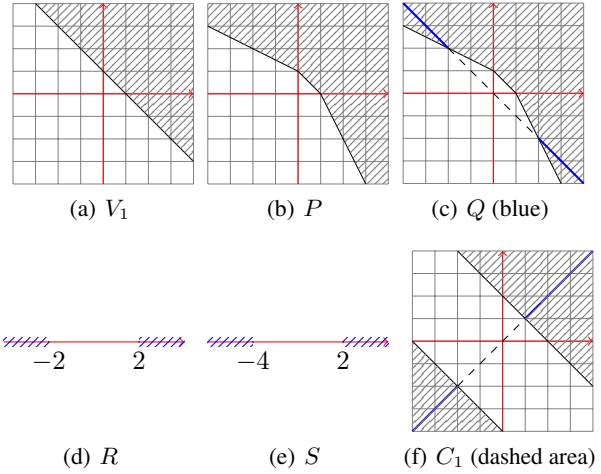
\begin{figure}[ht]
\begin{center}
    \subfigure[$V_1$]{
    \begin{tikzpicture}[scale=0.3]
	\draw[step=1cm,gray,very thin] (-4,-4) grid (4,4);
	\draw[->,red] (-4,0) -- (4,0);
	\draw[->,red] (0,-4) -- (0,4);
	\draw[-] (-3,4)--(4,-3);
	\fill[pattern=north east lines, pattern color=gray] (-3,4)--(4,4)--(4,-3);
    \end{tikzpicture}}
    \subfigure[$P$]{
    \begin{tikzpicture}[scale=0.3]
	\draw[step=1cm,gray,very thin] (-4,-4) grid (4,4);
	\draw[->,red] (-4,0) -- (4,0);
	\draw[->,red] (0,-4) -- (0,4);
	\fill[pattern=north east lines, pattern color=gray] (-3,4)--(4,4)--(4,-3);
	
	\draw[-] (0,1)--(-4,3);
	\draw[-] (1,0)--(3,-4);
	\draw[-] (1,0)--(0,1);
	\fill[pattern=north east lines, pattern color=gray] (0,1)--(-4,3)--(-4,4)--(-3,4)--cycle;
	\fill[pattern=north east lines, pattern color=gray] (1,0)--(3,-4)--(4,-4)--(4,-3)--cycle;
    \end{tikzpicture}}
    \subfigure[$Q$ (blue)]{
    \begin{tikzpicture}[scale=0.3]
	\draw[step=1cm,gray,very thin] (-4,-4) grid (4,4);
	\draw[->,red] (-4,0) -- (4,0);
	\draw[->,red] (0,-4) -- (0,4);
	\fill[pattern=north east lines, pattern color=gray] (-3,4)--(4,4)--(4,-3);

	\draw[-] (0,1)--(-4,3);
	\draw[-] (1,0)--(3,-4);
	\draw[-] (1,0)--(0,1);
	\fill[pattern=north east lines, pattern color=gray] (0,1)--(-4,3)--(-4,4)--(-3,4)--cycle;
	\fill[pattern=north east lines, pattern color=gray] (1,0)--(3,-4)--(4,-4)--(4,-3)--cycle;

	\draw[dashed] (-2,2)--(2,-2);
	\draw[-,blue,thick] (-2,2)--(-4,4);
	\draw[-,blue,thick] (2,-2)--(4,-4);
    \end{tikzpicture}}
    \\
    \subfigure[$R$]{
    \begin{tikzpicture}[scale=0.3]
	\draw[step=1cm,white,very thin] (-4,-4) grid (4,4);
	\draw[->,red] (-4,0) -- (4,0);
	\foreach \x in {-2,2}
	    \draw (\x cm,1pt)--(\x cm,-1pt) node[anchor=north] {$\x$};
	\fill[pattern=north east lines, pattern color=blue] (-4,-0.2) -- (-2,-0.2) -- (-2,0.2) -- (-4,0.2) -- cycle;
	\fill[pattern=north east lines, pattern color=blue] (4,-0.2) -- (2,-0.2) -- (2,0.2) -- (4,0.2) -- cycle;
    \end{tikzpicture}}
    \hspace{1pt}
    \subfigure[$S$]{
    \begin{tikzpicture}[scale=0.3]
	\draw[step=1cm,white,very thin] (-4,-4) grid (4,4);
	\draw[->,red] (-4,0) -- (4,0);
	\draw (-2cm,1pt)--(-2cm,-1pt) node[anchor=north] {$-4$};
	\foreach \x in {2}
	    \draw (\x cm,1pt)--(\x cm,-1pt) node[anchor=north] {$\x$};
	\fill[pattern=north east lines, pattern color=blue] (-4,-0.2) -- (-2,-0.2) -- (-2,0.2) -- (-4,0.2) -- cycle;
	\fill[pattern=north east lines, pattern color=blue] (4,-0.2) -- (2,-0.2) -- (2,0.2) -- (4,0.2) -- cycle;
    \end{tikzpicture}}
    \hspace{1pt}
    \subfigure[$C_1$ (dashed area)]{
    \begin{tikzpicture}[scale=0.3]
	\draw[step=1cm,gray,very thin] (-4,-4) grid (4,4);
	\draw[->,red] (-4,0) -- (4,0);
	\draw[->,red] (0,-4) -- (0,4);
	
	\draw[dashed] (-2,-2)--(1,1);
	\draw[-,blue,thick] (1,1)--(4,4);
	\draw[-,blue,thick] (-2,-2)--(-4,-4);
	\draw[-] (-2,4)--(4,-2);
	\draw[-] (-4,0)--(0,-4);
	\fill[pattern=north east lines, pattern color=gray] (-2,4)--(4,4)--(4,-2)--cycle;
	\fill[pattern=north east lines, pattern color=gray] (-4,0)--(-4,-4)--(0,-4)--cycle;	
    \end{tikzpicture}}
\end{center}
\caption{Construction steps of the decision region for a non-pyramidal network.}
\label{fig:non_pyramidal}
\end{figure}

Let $    V_1=\Setbar{(y_1,y_2)\in\RR^2}{y_1+y_2-1>0}$
then the decision region of the first class $C_1$ can be computed recursively:
\begin{align*}
    &P:=\hat{\sigma}_2^{-1}(V_1), \\
    &Q:= P\cap\textrm{range}(W_2^T), \\
    &R:= {(W_2W_2^T)}^{-1} W_2 (Q-b_2), \\
    &S:= {\hat{\sigma}_1}^{-1}(R), \\
    &C_1= W_1 {(W_1^TW_1)}^{-1} (S-b_1) + \textrm{ker}(W_1^T) .
\end{align*}
The outcome of each step is illustrated in Figure \ref{fig:non_pyramidal}.
One can clearly check that the decision region of the first class $C_1$ is given by 
\begin{align*}
    C_1 = \Setbar{x\in\RR^2}{x_1+x_2-2>0 \textrm{ and } x_1+x_2+4<0}
\end{align*}
which is thus a disconnected set.

Overall, this counter-example shows that pyramidal structure of the network, 
at least under the other conditions of Theorem 3.10,  
is a necessary condition to get connected decision regions.

\section*{Does the result hold for ReLU activation function?}
In the following, we show how to derive the decision regions of the network given in Equation \eqref{eq:net_relu}.
\begin{figure*}[ht]
\begin{center}
    \subfigure[$V_1$]{\includegraphics[width=0.3\linewidth]{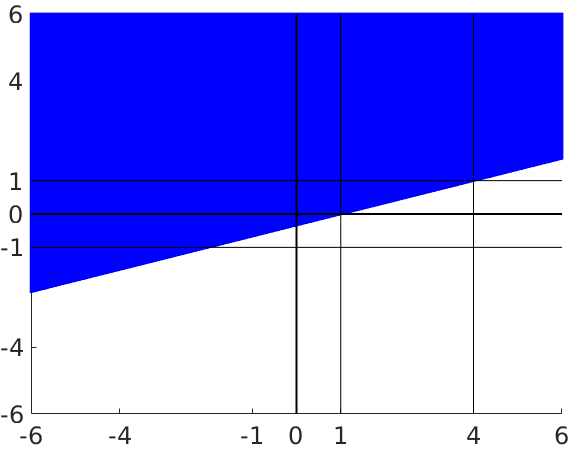}} \hspace{10pt}
    \subfigure[$P$]{\includegraphics[width=0.3\linewidth]{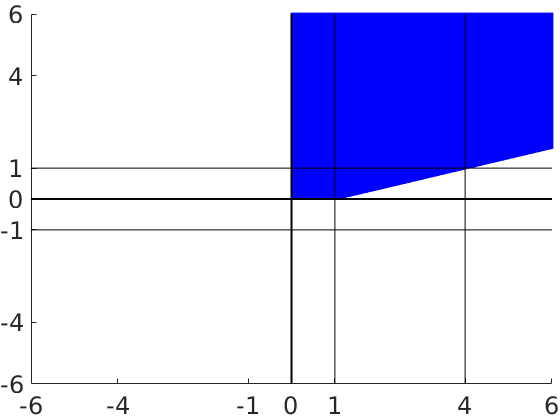}}  \hspace{10pt}
    \subfigure[$Q$]{\includegraphics[width=0.3\linewidth]{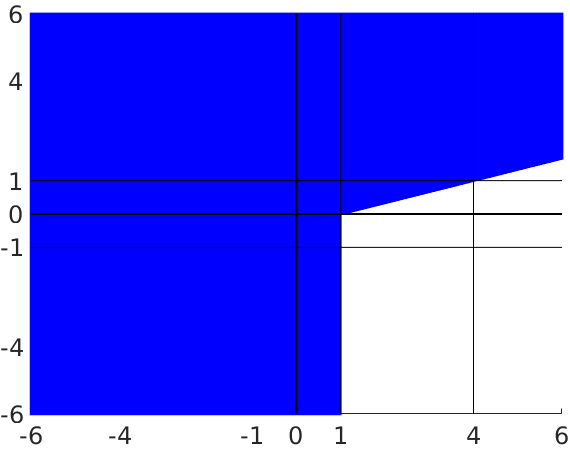}}
    \\
    \subfigure[$R$]{\includegraphics[width=0.3\linewidth]{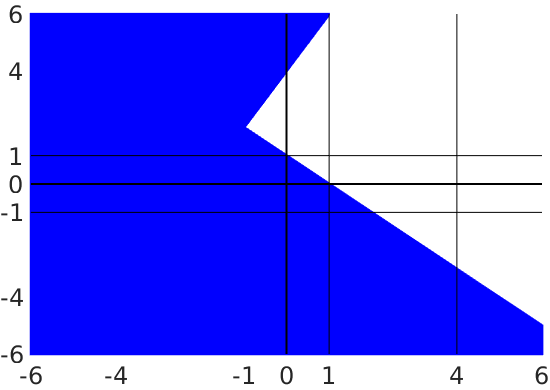}} \hspace{10pt}
    \subfigure[$S$]{\includegraphics[width=0.3\linewidth]{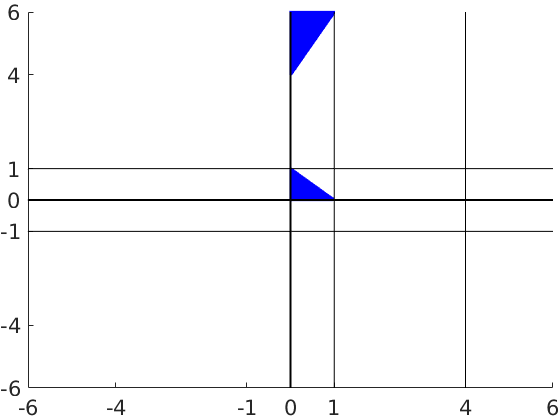}} \hspace{10pt}
    \subfigure[$C_1$]{\includegraphics[width=0.3\linewidth]{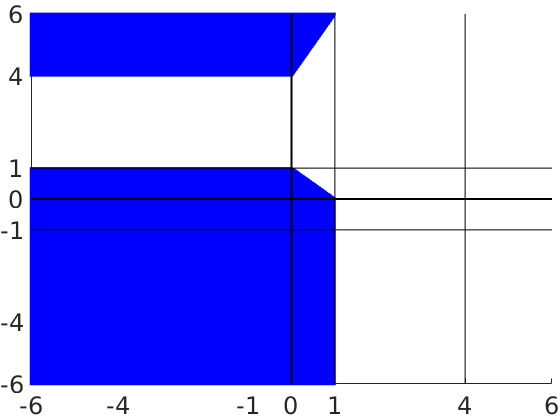}}
\end{center}
\caption{Construction of the decision region of the ReLU network given in Equation (5).}
\label{fig:relu_network}
\end{figure*}

Let us define the set 
\begin{align*}
    V_1
    &=\Setbar{y\in\RR^2}{(W_3^T y+b_3)_1 > (W_3^T y+b_3)_2} \\
    &=\Setbar{(y_1,y_2)\in\RR^2}{y_1-3y_2-1<0} .
\end{align*}
Note that the inverse mapping of the ReLU function $\sigma(t)=\max\Set{t,0}$ is given as
\begin{align*}
    \sigma^{-1}(t)=
    \begin{cases}
	t & t> 0\\
	\Setbar{x\in\RR}{x\leq 0} & t=0.
    \end{cases}
\end{align*}
The decision region of the first class $C_1$ can be computed recursively as 
(see Figure \ref{fig:relu_network} for the illustration):
\begin{align*}
    &P:=V_1\cap\range(\hat{\sigma}_2) = V_1\cap\RR^2_{+} \\
    &Q:=\hat{\sigma}_2^{-1}(P), \\
    &R:= {(W_2^T)}^{-1} (Q-b_2), \\
    &S:= R\cap\range(\hat{\sigma}_1) = R\cap\RR^2_{+} \\
    &C_1:= {\hat{\sigma}_1}^{-1}(S)
\end{align*}
By following these steps, one can easily check that
\begin{align*}
    C_1 
    &= \Setbar{x\in\RR^2} {x_1<1,\;x_2<1,\;x_1+x_2<1} 
    \\ 
    &\;\cup \Setbar{x\in\RR^2}{x_2>4,\;2x_1-x_2+4<0} 
\end{align*}
which is a disconnected set as illustrated in Figure \ref{fig:relu_network}.
\fi
\end{document}